\documentclass[letterpaper]{article} 
\usepackage{aaai2026}  
\usepackage{times}  
\usepackage{helvet}  
\usepackage{courier}  
\usepackage[hyphens]{url}  
\usepackage{graphicx} 
\usepackage{natbib}  
\usepackage{caption} 
\usepackage{algorithm}
\usepackage{algorithmic}
\usepackage{amsmath,amsfonts,amssymb}
\usepackage{empheq}

\usepackage{newfloat}
\usepackage{listings}

\usepackage{pifont}
\usepackage{booktabs}

\usepackage{amsthm}
\newtheorem{assumption}{Assumption} 

\theoremstyle{plain}
\newtheorem{theorem}{Theorem}
\newtheorem{proposition}{Proposition}

\theoremstyle{definition}
\newtheorem{definition}{Definition}
\theoremstyle{remark}

\theoremstyle{plain}

\DeclareCaptionStyle{ruled}{labelfont=normalfont,labelsep=colon,strut=off} 
\floatstyle{ruled}
\newfloat{listing}{tb}{lst}{}
\floatname{listing}{Listing}
\nocopyright 

\title{Peer Effect Estimation in the Presence of Simultaneous Feedback and Unobserved Confounders}

\author {
    Xiaojing Du\textsuperscript{\rm 1},
    Jiuyong Li\textsuperscript{\rm 1},   
    Lin Liu\textsuperscript{\rm 1},
    Debo Cheng\textsuperscript{\rm 1},
    Thuc.Le\textsuperscript{\rm 1}
}
\affiliations {
    \textsuperscript{\rm 1}University of South Australia\\
    \{xiaojing.du\}@mymail.unisa.edu.au,
    \{jiuyong.li, liu.lin, debo.cheng, Thuc.Le\}@unisa.edu.au
}

\begin{document}

\maketitle

\begin{abstract}

Estimating peer causal effects within complex real-world networks such as social networks is challenging, primarily due to simultaneous feedback between peers and unobserved confounders. Existing methods either address unobserved confounders while ignoring the simultaneous feedback, or account for feedback but under restrictive linear assumptions, thus failing to obtain accurate peer effect estimation. In this paper, we propose \textbf{DIG2RSI}, a novel \underline{\textbf{D}}eep learning framework which leverages \underline{\textbf{I}}-\underline{\textbf{G}} transformation (matrix operation) and \underline{\textbf{2SRI}} (an instrumental variable or IV technique) to address both simultaneous feedback and unobserved confounding, while accommodating complex, nonlinear and high-dimensional relationships. DIG2RSI first applies the \underline\bf{I}-\underline\bf{G} transformation to disentangle mutual peer influences and eliminate the bias due to the simultaneous feedback. To deal with unobserved confounding, we first construct valid IVs from network data. In stage 1 of 2RSI, we train a neural network on these IVs to predict peer exposure, and extract residuals as proxies for the unobserved confounders. In the stage 2, we fit a separate neural network augmented by an adversarial discriminator that incorporates these residuals as a control function and enforces the learned representation to contain no residual confounding signal. The expressive power of deep learning models in capturing complex non-linear relationships and adversarial debiasing enhances the effectiveness of DIG2RSI in eliminating bias from both feedback loops and hidden confounders. We prove consistency of our estimator under standard regularity conditions, ensuring asymptotic recovery of the true peer effect. Empirical results on two semi-synthetic benchmarks and a real-world dataset demonstrate that DIG2RSI outperforms existing approaches.

\end{abstract}


\section{Introduction}

Accurate estimation of \textbf{\textit{peer effects}} within networks is crucial in disciplines such as sociology~\cite{forastiere2021identification,lorenz2020social}, economics~\cite{bramoulle2009identification,ali2023peer}, and public health~\cite{liu2014identification,arcidiacono2005peer}. For instance, understanding how friends influence each other's health behaviors, such as exercise habits, smoking, or dietary choices, is vital for formulating effective public health policies. However, conducting causal inference in network settings presents two key challenges. First, \textbf{\textit{simultaneous feedback}} loops arise because an individual’s behavior both influences and is influenced by their peers at the same time, creating mutual causation that violates standard exogeneity assumptions~\cite{manski1993identification,ogburn2020causal}. Second, \textbf{\textit{unobserved confounders}} like environmental exposures or socioeconomic traits can affect multiple connected individuals at once, inducing spurious correlations~\cite{bramoulle2009identification,kallus2018removing}.

\begin{figure}[t]
	\centering
	\includegraphics[scale=0.18]{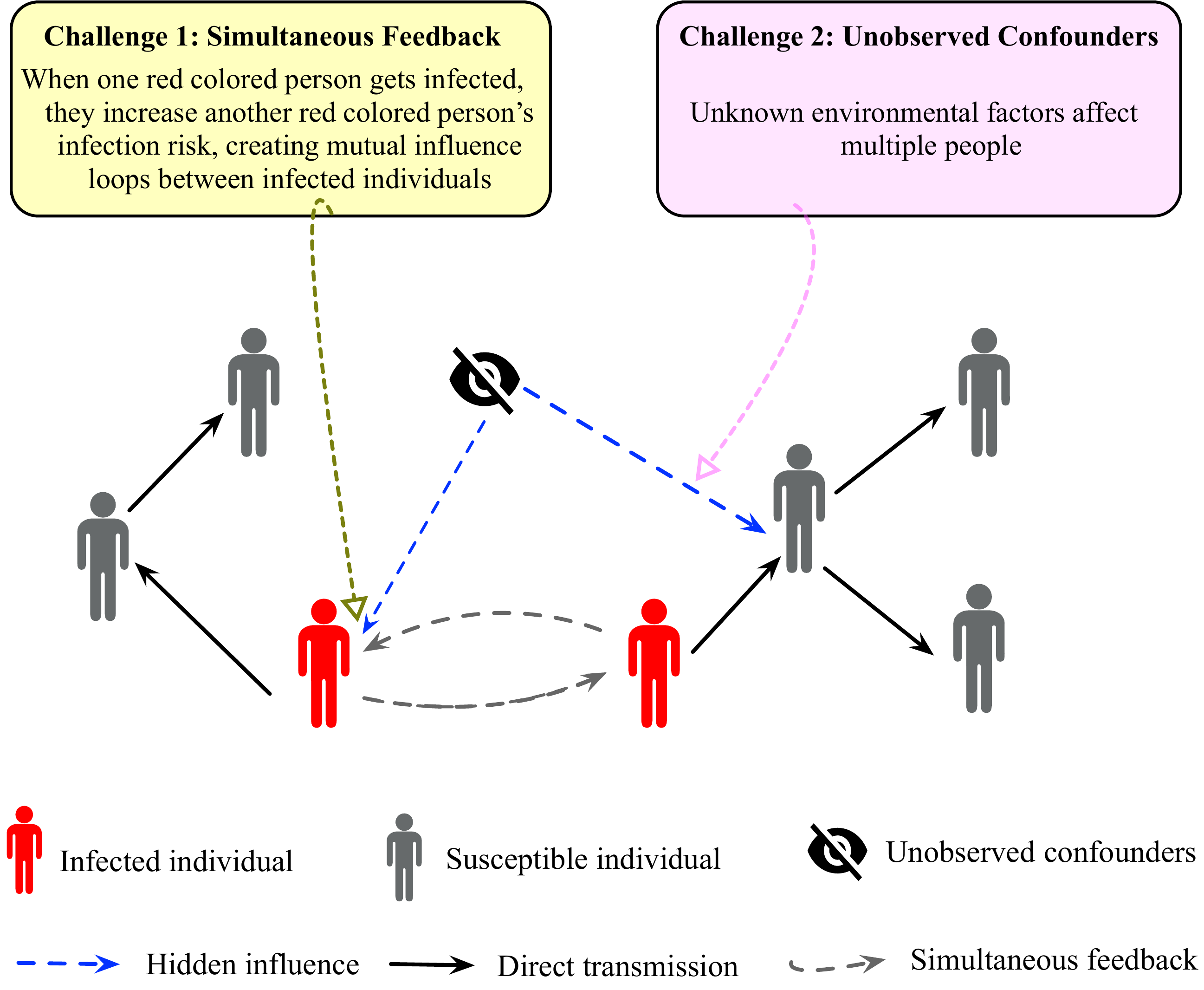}
	\caption{Illustrative example of the two core challenges in peer effect estimation in network setting: (a) simultaneous feedback loops among infected individuals, and (b) unobserved confounders affecting multiple individuals.} 
	\label{fig:1}
\end{figure}

To better illustrate the two challenges, we take the spread of infectious diseases as an example. In a community, when an individual becomes infected, they increase the likelihood of infection for their contacts; as those contacts in turn become infected, they feed back into the original individual’s risk, creating~\textit{simultaneous feedback}, as shown in Fig.~\ref{fig:1}. Furthermore, real-world transmission patterns are influenced by unobserved confounders such as external environment, further increasing the complexity of causal inference.


Instrumental Variables (IVs)~\cite{angrist1996identification,bollen2012instrumental,sargan1958estimation} provide an effective solution unobserved confounders.~\citet{bramoulle2009identification,von2019use,jochmans2023peer} used the average features of peers as IVs to identified peer effects in social networks using a \textit{linear-in-means} model. However, in reality, the effect of peers on an individual is often far more complex than a simple linear function of average peer behavior.

To address the linearity limitation,~\citet{boucher2024toward} proposed a peer effects model based on a Constant Elasticity of Substitution (CES) function with flexible parameters. However, using a CES function to model the social norm places strong restrictions on the functional form of the social norm and cannot accommodate potentially more complex peer‐reference mechanisms.


A summary of previous work on peer effects in the presence of simultaneous feedback and unobserved confounders is presented in Table~\ref{tab:peer-effects-comparison}. We observe a critical gap: existing methods mostly focus on tackling unobserved confounders without considering simultaneous feedback, and when they do consider both, they typically rely on restrictive linear assumptions, failing to handle nonlinear and high-dimensional dependencies simultaneously.

To address the research gap, we propose \textbf{DIG2RSI}, a novel \underline{\textbf{D}}eep learning framework which employs \underline{\textbf{I}}-\underline{\textbf{G}} transformation and the idea of \underline{\textbf{2SRI}} (two-stage residual inclusion) to deal with both simultaneous feedback and unobserved confounders at the same time, while accommodating complex, nonlinear, and high-dimensional network settings. \textbf{\textit{First}}, to disentangle and eliminate bias arising from simultaneous feedback loops, we introduce a specially designed differencing operation via the {\textbf{I}}-{\textbf{G}}~\cite{lesage2009introduction} transformation that “pre-whitens” the network before outcome modeling, thereby clarifying causal influence among peers. \textbf{\textit{Second}}, to cope with unobserved confounders, as well as the complex nonlinear relations and high-dimensionality of network data, we leverage the idea of the IV technique 2SRI~\cite{terza2008two} and the representation learning power of deep learning models. \textbf{\textit{Specifically}}, we construct valid instruments using second-order peer features, since they influence individual outcomes only through their effects on immediate peers and remain uncorrelated with unobserved confounders. We then develop a deep learning framework with two neural networks, one for regressing peer exposure on the constructed IVs and extracting residuals; and another for outcome prediction based using the residuals as control functions. Third, to further guard against any remaining confounding signal, we augment the second neural network with an adversarial discriminator that forces the learned representation to be orthogonal to the first-stage residuals. \textbf{\textit{Finally}}, our framework naturally accommodates continuous treatment variables, enabling a more realistic and flexible characterization of peer influence. To our knowledge, no prior work has jointly handled simultaneous feedback and unobserved confounders in nonlinear, high-dimensional network data as DIG2RSI does.

\begin{table}[!t]
  \centering
  \scriptsize
  \setlength{\tabcolsep}{3.5pt}   
  \begin{tabular}{@{}lccc@{}}
    \toprule
    \textbf{Representative Work} &
    \textbf{\begin{tabular}[c]{@{}c@{}}Simultaneous\\Feedback\end{tabular}} &
    \textbf{\begin{tabular}[c]{@{}c@{}}Unobserved\\Confounders\end{tabular}} &
    \textbf{Nonlinear} \\
    \midrule
    \citet{bramoulle2009identification} & \(\boldsymbol{\checkmark}\) & \(\boldsymbol{\checkmark}\) & \(\boldsymbol{\times}\) \\
    \citet{von2019use}                  & \(\boldsymbol{\times}\)     & \(\boldsymbol{\checkmark}\) & \(\boldsymbol{\times}\) \\
    \citet{jochmans2023peer}            & \(\boldsymbol{\times}\)     & \(\boldsymbol{\checkmark}\) & \(\boldsymbol{\times}\) \\
    \citet{boucher2024toward}           & \(\boldsymbol{\times}\)     & \(\boldsymbol{\checkmark}\) & \(\boldsymbol{\checkmark}\) \\
    Ours                                & \(\boldsymbol{\checkmark}\) & \(\boldsymbol{\checkmark}\) & \(\boldsymbol{\checkmark}\) \\
    \bottomrule
  \end{tabular}
  \caption{Summary of representative work on peer effect estimation.}
  \label{tab:peer-effects-comparison}
\end{table}

Our key contributions are as follows:
\begin{itemize}
  \item We introduce a tailored differencing operator to eliminate \textbf{simultaneous feedback loops}, enhancing the identifiability and accuracy of peer effect estimation.
\item We develop a deep-learning DIG2RSI framework for estimating peer effects with \textbf{unobserved confounders} in network data with nonlinear relationships, which is also applicable to continuous treatments.
  \item We provide a theoretical proof of consistency for the resulting estimator under standard regularity conditions.
  \item We demonstrate the effectiveness and robustness of IV‑2SRI on two semi‑synthetic benchmarks and a real‑world dataset.
\end{itemize}
\section{Related Work}

In network data, individuals' behaviors (outcomes) often influence each other, leading to the presence of peer effects~\cite{li2019randomization,gu2025peer,an2015instrumental}. In real-world scenarios, unobserved confounders may affect individual behaviors, violating the unconfoundedness assumption and thereby impacting the accuracy of causal inference. In this work, we categorize the estimation of peer effects in network data into two groups: those based on the unconfoundedness assumption and those that do not rely on this assumption~\cite{pearl2009causality,de2014testing}.

\textbf{Based on the Unconfoundedness Assumption.}
To estimate peer effects, \citet{ogburn2020causal} proposed a causal inference method based on chain graphs, demonstrating that under certain conditions, chain graph models can approximate the underlying causal directed acyclic graphs (DAGs). They applied this approach in both simulation studies and analyses of U.S. Supreme Court decision data. \citet{tchetgen2021auto} introduced the Auto-G-Computation method, which leverages chain graph models and Markov random field (MRF) structures, employing Gibbs sampling to compute network causal effects. Additionally, \citet{liu2020gmm} investigated the identification and estimation of peer effects in simultaneous equations network models, proposing a generalized method of moments (GMM) approach that combines linear and quadratic moment conditions to enhance estimation efficiency. \citet{bhattacharya2020causal} addressed situations where dependency structures in social networks are not fully known, proposing a causal inference method based on chain graphs. They utilized structure learning techniques and pseudolikelihood scoring methods to infer potential network structures from data. However, the above research typically assume the absence of unobserved confounders. In real-world contexts, this assumption often does not hold, and ignoring unobserved confounders can lead to biased causal estimates. 

\textbf{Without the Unconfoundedness Assumption.} To address this issue, researchers have proposed various methods. For instance,~\citet{liu2014identification} explored the interdependence of individual decision-making in social networks. They analyzed the use of Bonacich centrality as IVs within the Simultaneous Equations Network Models (SENM) framework to mitigate the impact of endogeneity on estimation results. However, their approach assumes that an individual's behavior is only influenced by direct friends, neglecting higher-order network effects (such as second- or third-order neighbors). ~\citet{o2014estimating} proposed an IV-based method to estimate peer effects in longitudinal dyadic data. Their approach utilizes genes (Mendelian Randomization) as IVs, combined with DAG analysis and two-stage least squares (2SLS), to address latent homophily and unobserved confounding. However, this study is primarily applicable to dyadic relationships and may not be directly generalizable to more complex social network structures.~\citet{bramoulle2009identification} and~\citet{von2019use} employed a mean linear model, using the average characteristics of second-order neighbors to address endogeneity issues in peer outcomes. However, assuming a linear relationship between individuals and their peers oversimplifies the inherent complexity of social interactions, potentially introducing bias in causal effect estimation.

\section{Preliminaries and Problem Setting}

In this section, we first introduce network data. We then present the underlying causal mechanism via a directed acyclic graph (DAG) (Fig.~\ref{fig:dag}), and justify the use of second‑order neighbor features as an IV. Finally, we define the peer effect formally and state our objective. Please refer to Appendix A for an introduction to the basic concepts related to IVs, and Appendix B for a summary of the symbols used in the paper.

\subsection{Network Data} We represent a social network using a graph $\mathcal{G} = (\mathcal{V}, \mathcal{E})$, where $\mathcal{V}$ is the set of $N$ nodes (units or individuals) and $\mathcal{E}$ is the set of edges (relationships). The network to be non-directional. Let \(\boldsymbol{G}\) denote the row-normalized adjacency matrix of the graph. For each node $i \in \mathcal{V}$, we observe a set of exogenous features $X_i$ and an outcome $Y_i$.

\begin{figure}[t]
	\centering
	\includegraphics[scale=0.3]{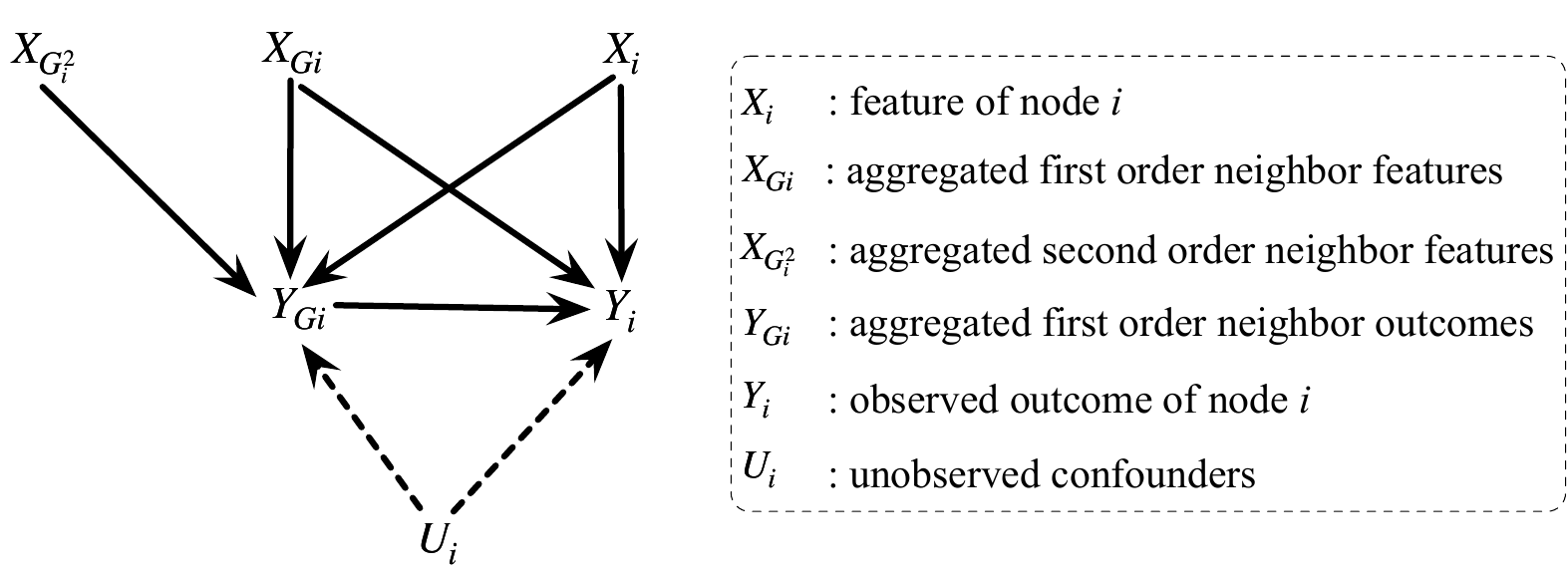}
	\caption{The causal DAG of our research problem. } 
	\label{fig:dag}
\end{figure}

\subsection{Causal DAG and Instrumental Variable}
To model the data generating process, we assume the causal DAG in Fig.~\ref{fig:dag}, which explicitly shows how a unit is influenced by others in the network. In the DAG, \(Y_{G i}\) denote the aggregated peer outcome that affects unit \(i\)'s outcome \(Y_i\), and \(U_i\) is an unobserved confounder influencing both \(Y_{G i}\) and \(Y_i\). \(Y_i\) is affected by its own features \(X_i\), first-order neighbor features \(X_{G_i}\), and the aggregated first-order neighbor outcome \(Y_{G i}\), while second-order neighbor features \(X_{G^{2}_{\mkern0.5mu i}}\) do not directly affect \(Y_i\).

To identify the peer effect \(Y_{G_i} \to Y_i\) in the presence of unobserved confounder $U_i$, we follow~\citet{bramoulle2009identification} and use the second-order neighbor features \(X_{G^{2}_{\mkern0.5mu i}}\) as an instrument. Intuitively, \(X_{G^{2}_{\mkern0.5mu i}}\) is influenced by the features of its neighbors, which correspond to \(i\)'s own features and those of its second-order neighbors. The second-order features \(X_{G^{2}_{\mkern0.5mu i}}\) affect \(Y_i\) only indirectly through \(Y_{G i}\) and do not have a direct effect on \(Y_i\); hence, \(X_{G^{2}_{\mkern0.5mu i}}\) serves as a valid instrument for \(Y_{G i}\) in the estimation of its effect on \(Y_i\). For a formal proof of the validity of \(X_{G^{2}_{\mkern0.5mu i}}\) being an IV, please see~\citet{bramoulle2009identification}

In practice, to handle non-linearities, the following assumption in addition to the assumptions for the standard IV (presented in Appendix A) is necessary.

\begin{assumption}[Monotonicity~\cite{angrist1995identification}]\label{ass:monotonicity}
For every individual \(i\) and any two instrumental values \(x_{G^{2}_{\mkern0.5mu i}} \le x_{G^{2}_{\mkern0.5mu i}}'\), the induced treatment levels satisfy
\begin{equation}
  Y_{Gi}(x_{G^{2}_{\mkern0.5mu i}}) \;\le\; Y_{Gi}(x_{G^{2}_{\mkern0.5mu i}}').
  \label{eq:monotonicity}
\end{equation}
\end{assumption}

\subsection{Problem Setting}

\begin{definition}[Peer Effect (PE)] Based on the causal DAG in Fig.~\ref{fig:dag}, PE can be defined as the rate of change in an individual $i$'s outcome, i.e., \(Y_i\) with an infinitesimal change in $i$'s neighbourhood outcome \(Y_{G i}\). That is
\begin{equation}
\mathrm{PE}
= \frac{d}{dy_{G_i}}\mathbb{E}\bigl[Y_i \,\big|\, \mathrm{do}(Y_{G_i} = y_{G_i})\bigr]
\label{eq:peer-effect}
\end{equation}
\end{definition}
\noindent where \(y_{G_i}\) is a specific value of \(Y_{G_i}\), and \(\mathrm{do}(Y_{G_i}=y_{G_i})\) means that value of \(Y_{G_i}\) is set to \(y_{G_i}\)~\cite{pearl2009causality}. When \(Y_i\) is modelled by a structural equation model, the coefficient of \(Y_{G_i}\) in the structural equation is a consistent estimate of PE~\cite{forastiere2024causal}.

\textbf{Objective.} We aim at estimating the PE in network data in the presence of simultaneous feedback and unobserved confounders.

\section{Method}

\label{lab:method}

As discussed in the Introduction, the inherent reciprocal nature and the unobserved confounders involved in peer interactions present two significant challenges to peer effect estimation with observational network data, in addition to the complex nonlinear relations. In this section, we present our DIG2RSI method (as illustrated in Fig.~\ref{fig:flow}) for dealing with these challenges, including the part of adopting \(\boldsymbol{I}\!-\!\boldsymbol{G}\) transformation for coping with feedback loop, and a novel deep learning framework for dealing with unobserved confounding and non-linearity supported by the IV technique 2SRI.

\subsection{Addressing Simultaneous Feedback}
The \textbf{\textit{first core challenge}} is that \(Y_{G i}\) and \(Y_i\) exhibit simultaneous feedback, which introduces bias when estimating peer effects.

\textbf{Structural Equations.} Based on the DAG Fig.~\ref{fig:dag}, we use a Structural Equation Model (SEM) to describe the data-generating process. For an individual \(i\), we have
\begin{equation}
  Y_i = \beta_i\,Y_{G i} + \gamma_i\,X_{G i} + \delta_i\,X_i + \lambda_i\,U_i + \varepsilon_i.
  \label{eq:structural-scalar}
\end{equation}
where \(\varepsilon_i\) denotes the noise term.. 

In vector/matrix form, letting \(\boldsymbol{Y}, \boldsymbol{X}, \boldsymbol{U}, \boldsymbol{\varepsilon}\) collect the corresponding quantities across all units, the SEM becomes
\begin{equation}\label{eq:prop-structural}
\begin{gathered}
  \boldsymbol{Y}
  = \boldsymbol{\beta}\,\boldsymbol{G}\,\boldsymbol{Y}
    + \boldsymbol{\gamma}\,\boldsymbol{G}\,\boldsymbol{X}
    + \boldsymbol{\delta}\,\boldsymbol{X}
    + \boldsymbol{\lambda}\,\boldsymbol{U}
    + \boldsymbol{\varepsilon},\\
  \mathbb{E}\bigl[\boldsymbol{\varepsilon}\mid \boldsymbol{X},\boldsymbol{U}\bigr]
  = \boldsymbol{0}.
\end{gathered}
\end{equation}
where \(\boldsymbol{\beta}, \boldsymbol{\gamma}, \boldsymbol{\delta}, \boldsymbol{\lambda}\) are coefficient parameters. Simultaneous feedback occurs because \(\boldsymbol{Y}\) appears on both sides of the structural equation.

\textbf{\(\boldsymbol{I}\!-\!\boldsymbol{G}\) Transformation.} 
To eliminate the bias caused by simultaneous feedback in networks, we adopt the differencing operation, known as \(\boldsymbol{I}\!-\!\boldsymbol{G}\) transformation. 

\begin{figure*}[t]
	\centering
	\includegraphics[scale=0.23]{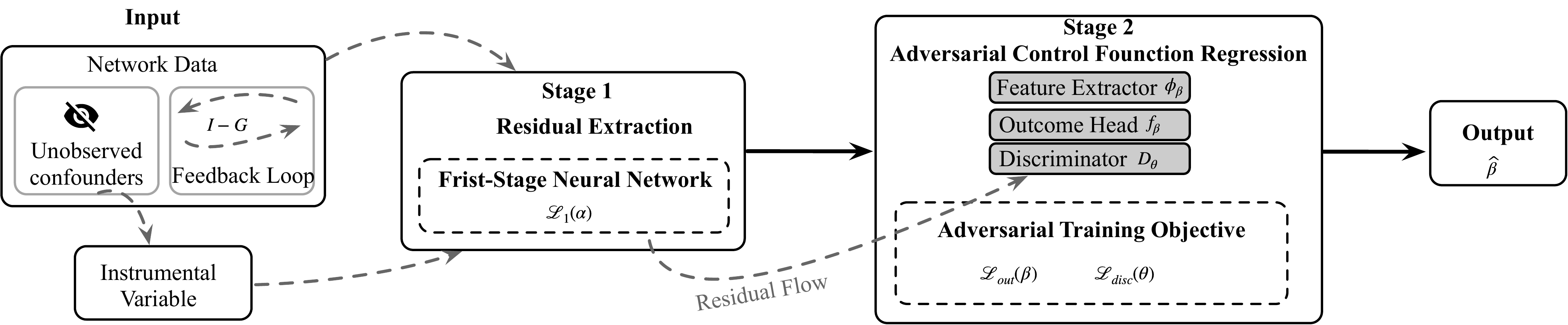}
	\caption{An overview of DIG2RSI, a novel \textbf{D}eep learning framework which employs \(\boldsymbol{I}\!-\!\boldsymbol{G}\) transformation for eliminating simultaneous feedback and embeds the idea of \textbf{2RSI} approach for dealing with unobserved confounders. } 
	\label{fig:flow}
\end{figure*}
\begin{proposition}[Elimination of Simultaneous Feedback]\label{prop:no-feedback}
Under the SEM in Equation~\eqref{eq:prop-structural}, assume \(\lvert\boldsymbol{\beta}\rvert<1\) and the spectral radius \(\rho(\boldsymbol{G})\le1\). The feedback loop can be eliminated via the \(\boldsymbol{I}\!-\!\boldsymbol{G}\) transformation.
\end{proposition}
\begin{proof}

Re-arranging Equation~\eqref{eq:prop-structural} by moving the feedback term to the left gives:
\begin{equation}\label{eq:proof-move}
  (\boldsymbol{I} - \boldsymbol{\beta}\,\boldsymbol{G})\,\boldsymbol{Y}
  = \boldsymbol{\gamma}\,\boldsymbol{G}\,\boldsymbol{X}
    + \boldsymbol{\delta}\,\boldsymbol{X}
    + \boldsymbol{\lambda}\,\boldsymbol{U}
    + \boldsymbol{\varepsilon}.
\end{equation}
\(\boldsymbol{I}\) is the identity matrix.

Assuming \(\lvert\boldsymbol{\beta}\rvert<1\) and \(\rho(\boldsymbol{G})\le1\), then \(\boldsymbol{I}-\boldsymbol{\beta}\boldsymbol{G}\) is invertible, so
\begin{equation}\label{eq:proof-solve}
  \boldsymbol{Y}
  = (\boldsymbol{I} - \boldsymbol{\beta}\,\boldsymbol{G})^{-1}
    \bigl[
      \boldsymbol{\gamma}\,\boldsymbol{G}\,\boldsymbol{X}
      + \boldsymbol{\delta}\,\boldsymbol{X}
      + \boldsymbol{\lambda}\,\boldsymbol{U}
      + \boldsymbol{\varepsilon}
    \bigr].
\end{equation}
For consistent estimation of PE in the presence of the network-induced dependencies~\cite{manski1993identification}, we need a version of the structural equation. Differencing both sides with \((\boldsymbol{I} - \boldsymbol{G})\) yields:
\begin{equation}\label{eq:proof-diff}
\begin{split}
  (\boldsymbol{I} - \boldsymbol{G})\,\boldsymbol{Y}
  &= (\boldsymbol{I} - \boldsymbol{G})\,(\boldsymbol{I} - \boldsymbol{\beta}\,\boldsymbol{G})^{-1}\\
  &\quad\times
    \bigl[
      \boldsymbol{\gamma}\,\boldsymbol{G}\,\boldsymbol{X}
      + \boldsymbol{\delta}\,\boldsymbol{X}
      + \boldsymbol{\lambda}\,\boldsymbol{U}
      + \boldsymbol{\varepsilon}
    \bigr].
\end{split}
\end{equation}
Rearranging gives the final “no-feedback” form:
\begin{equation}\label{eq:no-feedback-final}
\begin{split}
  (\boldsymbol{I} - \boldsymbol{G})\,\boldsymbol{Y}
  &= (\boldsymbol{I} - \boldsymbol{\beta}\,\boldsymbol{G})^{-1}
     \bigl[
       (\boldsymbol{\delta}\,\boldsymbol{I} + \boldsymbol{\gamma}\,\boldsymbol{G})\,(\boldsymbol{I}-\boldsymbol{G})\,\boldsymbol{X} \\
  &\quad\;+\,\boldsymbol{\lambda}\,(\boldsymbol{I}-\boldsymbol{G})\,\boldsymbol{U}
     + (\boldsymbol{I}-\boldsymbol{G})\,\boldsymbol{\varepsilon}
     \bigr].
\end{split}
\end{equation}
This shows that after these transformations the simultaneous feedback effect embedded in \(\boldsymbol{G}\boldsymbol{Y}\) is removed, yielding a representation suitable for consistent estimation.
\end{proof}

\begin{algorithm}[t!]
\caption{DIG2RSI}
\label{alg:iv2sri}
\textbf{Input:} Adjacency matrix \(\boldsymbol G\), features \(\boldsymbol X\), first-order features \(\boldsymbol {X_{G}}\); outcomes \(\boldsymbol Y\); instrument second-order features \(\boldsymbol {X_{G^2}}\); neural network classes for stage 1 and stage 2; learning rates \(\eta_1,\eta_2,\eta_{\text{disc}}\); adversarial weight \(\lambda_{\text{a}}\); epochs \(o,p\).\\
\textbf{Output:} Peer effect parameter \(\widehat{\boldsymbol\beta}\).
\begin{algorithmic}[1]
  \STATE \textbf{\(\boldsymbol{I}-\boldsymbol{G}\) Transformation} to remove feedback.\\
   \textbf{Stage 1: Estimate endogenous peer outcome}
  \STATE Initialize neural network \(r_{\boldsymbol\alpha}\).
  \FOR{epoch = 1 to \(o\)}
    \STATE \(\widehat{\boldsymbol {Y_{G}}} \gets r_{\boldsymbol\alpha}(\boldsymbol {X_{G^2}}, \boldsymbol {X_{G}}, \boldsymbol X)\)
    \STATE \(\mathcal{L}_1 \gets \frac{1}{n}\|\boldsymbol  {Y_{G}} - \widehat{\boldsymbol  {Y_{G}}}\|_2^2\)
    \STATE \(\boldsymbol\alpha \gets \boldsymbol\alpha - \eta_1 \nabla_{\boldsymbol\alpha} \mathcal{L}_1\)
  \ENDFOR
  \STATE Compute residual \(\widehat{\boldsymbol V} \gets \boldsymbol  {Y_{G}} - r_{\boldsymbol\alpha}(\boldsymbol {X_{G^2}}, \boldsymbol {X_{G}}, \boldsymbol X)\)

   \textbf{Stage 2: Adversarial Control Function Regression}
  \STATE Initialize outcome model parameters \(\boldsymbol\beta\) and discriminator parameters \(\boldsymbol\theta\).
  \FOR{epoch = 1 to \(p\)}
    \STATE Construct input \(\boldsymbol Z \gets (\boldsymbol  {Y_{G}}, \boldsymbol {X_{G}}, \boldsymbol X, \widehat{\boldsymbol V})\)
    \STATE Representation: \(\boldsymbol h \gets \phi_{\boldsymbol\beta}(\boldsymbol Z)\)  
    \STATE Outcome prediction: \(\widehat{\boldsymbol Y} \gets f_{\boldsymbol\beta}(\boldsymbol h)\)
    \STATE Discriminator prediction of residual: \(\widehat{\boldsymbol V}_{\text{disc}} \gets D_{\boldsymbol\theta}(\boldsymbol h)\)
    \STATE Compute losses: 
       $
         \mathcal{L}_{\text{out}}
         = \frac{1}{n}\|\widehat{\boldsymbol Y} - \boldsymbol Y\|_2^2,\quad
         \mathcal{L}_{\text{disc}}
         = \frac{1}{n}\|\widehat{\boldsymbol V}_{\text{disc}} - \widehat{\boldsymbol V}\|_2^2.
       $
    \STATE \textbf{(a) Discriminator update:} 
      \(\boldsymbol\theta \gets \boldsymbol\theta - \eta_{\text{disc}} \nabla_{\boldsymbol\theta} \mathcal{L}_{\text{disc}}\)
    \STATE \textbf{(b) Main model update:} 
      $
        \boldsymbol\beta \gets \boldsymbol\beta - \eta_2 \nabla_{\boldsymbol\beta}
        \bigl(\mathcal{L}_{\text{out}} - \lambda_{\text{a}}\mathcal{L}_{\text{disc}}\bigr)
      $
  \ENDFOR
  \STATE \textbf{Return:} \(\widehat{\boldsymbol\beta}\)
\end{algorithmic}
\end{algorithm}

\subsection{Addressing Unobserved Confounding and Nonlinearity}

To tackle the bias caused by unobserved confounders and the non-linear relations in network data, as well as removing the restriction of binary treatment, we propose a deep learning framework which embeds the idea the two-stage residual inclusion (2SRI)~\cite{terza2008two} and is equipped with the ability to cope with non-linear relations, high-dimensionality and non-binary treatment.

Specifically, as in illustrated in Fig.~\ref{fig:flow} and detailed in Algorithm~\ref{alg:iv2sri} (lines 2 to 17), we replace both linear regressions with neural networks, enabling the model to capture rich nonlinearities and and high-dimensional features. In \textbf{\textit{Stage 1}}, a small multi‐layer perceptron approximates the conditional expectation of the \(\boldsymbol{I}\!-\!\boldsymbol{G}\) transformed peer outcome using valid instruments, and the resulting residuals serve as control‐function proxies for unobserved confounders. \textbf{\textit{In Stage 2}}, these residuals are incorporated into a second neural outcome model; to eliminate any remaining confounding signal, we further augment this stage with an adversarial discriminator that forces the learned representation to be orthogonal to the first‐stage residuals. This two‐part architecture thus yields consistent, flexible estimates of peer effects under complex network.

As shown in Algorithm~\ref{alg:iv2sri}, before the two stages of 2RSI, first apply the \(\boldsymbol{I}\!-\!\boldsymbol{G}\) transformation to remove simultaneous feedback, as justified in Proposition~\ref{prop:no-feedback}. Define the preprocessed (no-feedback) variables
$
\widetilde{\boldsymbol{Y_G}} := (\boldsymbol{I}\!-\!\boldsymbol{G})\,\boldsymbol{Y_G},\quad
\widetilde{\boldsymbol{X_G}} := (\boldsymbol{I}\!-\!\boldsymbol{G})\,\boldsymbol{X_G},\quad
\widetilde{\boldsymbol{X_{G^2}}} := (\boldsymbol{I}\!-\!\boldsymbol{G})\,\boldsymbol{X_{G^2}},\quad
\widetilde{\boldsymbol U} := (\boldsymbol{I}\!-\!\boldsymbol{G})\,\boldsymbol U.
$ In the derivation below, for notational simplicity we drop the tildes and treat all variables as already preprocessed, i.e., every occurrence of \(\boldsymbol {Y_{G}},\boldsymbol {X_{G}},\boldsymbol {X_{G^2}},\boldsymbol U\) refers to the no-feedback version.

\subsubsection{\textbf{\textit{Stage 1 of DIG2RSI}: Residual Extraction.}} We train a small MLP (multi‐layer perceptron) to approximate the conditional expectation \({{Y}}_{G i} = r_{\boldsymbol\alpha}(X_{G^{2}_{\mkern0.5mu i}},\,X_{G i}, X_i\)), where \(r_{\alpha}\) has two hidden layers with ReLU activations, BatchNorm and Dropout to guard against overfitting.  We fit \(\boldsymbol\alpha\) by minimizing
\begin{equation}\label{eq:first_stage_loss}
\mathcal L_{1}(\boldsymbol\alpha)
= \frac{1}{n}\sum_{i=1}^n
    \bigl(r_{\alpha}(X_{G^{2}_{\mkern0.5mu i}},\,X_{G i},\,X_i) - \ Y_{G i}\bigr)^2.
\end{equation}

After convergence, we compute the control‐function residual
\begin{equation}\label{eq:first_residual}
  \widehat V_i
  =  Y_{G i} 
    \;-\; \widehat{ {Y}}_{G i}\,.
\end{equation}
which serves as a proxy for the unobserved confounder’s effect \(\omega_1 U_i\).

\subsubsection{\textbf{\textit{Stage 2 of DIG2RSI}: Adversarial Control Function Regression.}} Define the second-stage input vector \(Z_i^{(2)} = (Y_{G i}, X_{G i}, X_i, \widehat V_i)\). Our adversarial architecture comprises three modules:

\begin{itemize}
  \item A \emph{feature extractor} 
    \(\displaystyle \boldsymbol h_i = \phi_{\boldsymbol\beta}(Z_i^{(2)})\), implemented as the final hidden layer of a 2‑layer MLP on 
    \(Z_i^{(2)}\).  Its output \(\boldsymbol h_i\in\mathbb {R}^r\) is the learned node embedding.  
  \item An \emph{outcome head}  
    $
      f_{\beta}(Z_i^{(2)}) = w_{\rm out}^\top \boldsymbol h_i + b_{\rm out},
   $
    where \(w_{\rm out}\in\mathbb R^r\) is the learned weight vector and \(b_{\rm out}\in\mathbb R\) is the learned bias; together they map \(\boldsymbol h_i\) to the prediction of \(Y_i\).  
  \item A \emph{discriminator}  
    $
      D_{\theta}(\boldsymbol h_i) = w_{\rm disc}^\top \boldsymbol h_i + b_{\mathrm disc},
    $
    where \(w_{\rm disc}\in\mathbb R^r\) and \(b_{\rm disc}\in\mathbb R\) are learned, and \(D_{\theta}(\boldsymbol h_i)\) attempts to reconstruct the first‐stage residual \(\widehat V_i\) from \(\boldsymbol h_i\).
\end{itemize}

We optimize two losses:
\begin{align}
\mathcal L_{\rm out}(\beta)
&= \frac{1}{n}\sum_{i=1}^n
   \bigl(f_{\boldsymbol\beta}(Z_i^{(2)}) -Y_i\bigr)^2,
   \label{eq:adv_out}
\end{align}
\begin{equation}
\mathcal L_{\rm disc}(\theta)
= \frac{1}{n}\sum_{i=1}^n
   \bigl(D_{\theta}(\boldsymbol h_i) - \widehat V_i\bigr)^2.
\label{eq:adv_disc}
\end{equation}

The overall adversarial objective is the min–max problem
\begin{equation}\label{eq:adv_obj}
\min_{\boldsymbol\beta}\;\max_{\boldsymbol\theta}\;\Bigl[
  \mathcal L_{\rm out}(\boldsymbol\beta)
  \;-\;\lambda_{\rm a}\,
    \mathcal L_{\rm disc}(\boldsymbol\theta)
\Bigr].
\end{equation}
where \(\lambda_{\rm a}>0\) trades off outcome‐fit against confounder‐removal.  

\paragraph{Optimization procedure.} We alternate between:
\begin{enumerate}
  \item Discriminator step (\(\boldsymbol\theta\)‐update):  
    \(\boldsymbol\theta\leftarrow\boldsymbol\theta - \eta_{\rm disc}\,\nabla_{\boldsymbol\theta}\,\mathcal L_{\rm disc}(\boldsymbol\theta)\).  
  \item Main‐model step (\(\boldsymbol\beta\)‐update):  
    \(\boldsymbol\beta\leftarrow\boldsymbol\beta - \eta_{\rm 2}\,\nabla_{\boldsymbol\beta}
      \bigl[\mathcal L_{\rm out}(\boldsymbol\beta)
             - \lambda_{\rm a}\,\mathcal L_{\rm disc}(\boldsymbol\theta)\bigr].\)
\end{enumerate}

At convergence, the learned embedding \(\boldsymbol h_i\) is maximally predictive of the outcome but minimally predictive of the residual \(\widehat V_i\), ensuring that the control function successfully removes the influence of unobserved confounders.

\subsection{Consistency of the DIG2RSI Estimator}

Consistency of an estimator provides a fundamental guarantee for valid causal effect inference. In this subsection, we establish that the DIG2RSI estimator is consistent for the peer effect parameter under standard regularity and identification conditions.

\paragraph{Scalar Structural Form.} For each unit \(i\), after feedback removal, the data-generating process is
\begin{align}
  Y_{G i}
  &= \phi_i\,X_{G^{2}_{\mkern0.5mu i}}
    + \psi_i\,X_{G i}
    + \delta_i\,X_i
    + \omega_i\,U_i
    + \varepsilon_i^{(1)},\label{eq:scalar-first-stage}\\
  Y_i
  &= \beta_i\,Y_{G i}
    + \gamma_i\,X_{G i}
    + \delta_i\,X_i
    + \lambda_i\,U_i
    + \varepsilon_i^{(2)}.\label{eq:scalar-second-stage}
\end{align}

\paragraph{Matrix Structural Form.} Let \(\boldsymbol {Y_{G}} = \boldsymbol G \boldsymbol Y\), \(\boldsymbol {X_{G}} = \boldsymbol G \boldsymbol X\), and \(\boldsymbol {X_{G^2}} = \boldsymbol G^2 \boldsymbol X\), all understood to be the preprocessed (i.e., \(\boldsymbol{I}\!-\!\boldsymbol{G}\) transformed) versions. Then
\begin{align}
  \boldsymbol {Y_{G}}
  &= \boldsymbol \Phi\,\boldsymbol {X_{G^2}}
    + \boldsymbol \Psi\,\boldsymbol {X_{G}}
     + \boldsymbol \delta\,\boldsymbol X
    + \boldsymbol \omega\,\boldsymbol U
    + \boldsymbol \varepsilon^{(1)},\label{eq:matrix-first-stage}\\
  \boldsymbol Y
  &= \boldsymbol \beta\,\boldsymbol {Y_{G}}
    + \boldsymbol \gamma\,\boldsymbol {X_{G}}
    + \boldsymbol \delta\,\boldsymbol X
    + \boldsymbol \lambda\,\boldsymbol U
    + \boldsymbol \varepsilon^{(2)}.\label{eq:matrix-second-stage}
\end{align}
where \(\boldsymbol \Phi,\boldsymbol \Psi,\boldsymbol \omega\) are parameter matrices and \(\boldsymbol \varepsilon^{(1)},\boldsymbol \varepsilon^{(2)}\) are noise vectors.

\subsection{Error Bound Estimate}

In this part, we provide the two-stage error bound estimates.

\paragraph{\textbf{\textit{Stage 1 of DIG2RSI}: Approximation and Consistency of Residuals.}}

To obtain consistency of the control function, we need that the first-stage estimator approximates the true conditional expectation well. 

\begin{assumption}[Approximation and Estimation Accuracy]\label{ass:first-stage-accuracy}
Let \(\mathcal{F}_n=\{r(\cdot;\boldsymbol\alpha):\boldsymbol\alpha\in\Theta_n\}\) be the class of neural networks used to estimate 
\begin{equation}\label{eq:oracle-function}
\boldsymbol{r}^*(\boldsymbol {X_{G^2}},\boldsymbol{X_{G}},\boldsymbol{X})
=
\mathbb{E}\bigl[
  \boldsymbol{Y_{G}}
  \bigm|
  \boldsymbol {X_{G^2}},
  \boldsymbol{X_{G}}
  ,\boldsymbol{X}
\bigr].
\end{equation}
Assume:
\begin{enumerate}
  \item (\emph{Approximation}) For any \(\epsilon>0\) there exists \(n_0\) such that for all \(n\ge n_0\),
\begin{equation}\label{eq:approximation}
\begin{aligned}
    &\inf_{\boldsymbol\alpha\in\Theta_n}
    \mathbb{E}\bigl[\|\boldsymbol{r}^*(\boldsymbol {X_{G^2}},\boldsymbol{X_{G}},\boldsymbol{X}) \\
    & \qquad - \boldsymbol r(\boldsymbol {X_{G^2}},\boldsymbol{X_{G}},\boldsymbol{X};\boldsymbol\alpha)\|_2^2\bigr]
    < \epsilon.
\end{aligned}
\end{equation}
  \item (\emph{Estimation}) A uniform law of large numbers holds over \(\mathcal{F}_n\), so that the empirical risk~\cite{vapnik1999overview}
\begin{equation}\label{eq:empirical-risk}
    \widehat{\boldsymbol R}_{1,n}(\boldsymbol\alpha)
    = \frac{1}{n}\Big\|\boldsymbol{Y_{G}} - \boldsymbol r(\boldsymbol {X_{G^2}},\boldsymbol{X_{G}},\boldsymbol{X};\boldsymbol\alpha)\Big\|_2^2
  \end{equation}
  converges uniformly to its population counterpart, and the optimizer \(\widehat{\boldsymbol\alpha}_n = \arg\min_{\boldsymbol\alpha\in\Theta_n}\widehat{\boldsymbol R}_{1,n}(\boldsymbol\alpha)\) satisfies consistency toward the best approximant.
\end{enumerate}
\end{assumption}

\begin{proposition}[Stage 1 Residual Consistency]\label{prop:first-stage-residual}
Under Assumptions~\ref{ass:first-stage-accuracy}, the first-stage residual
\begin{equation}\label{eq:residual-def}
\widehat{\boldsymbol V}
= \boldsymbol {Y_{G}} - \boldsymbol r(\boldsymbol {X_{G^2}},\boldsymbol {X_{G}},\boldsymbol{X}; \widehat{\boldsymbol\alpha}_n)
\end{equation}
satisfies
\begin{equation}\label{eq:residual-consistency}
\widehat{\boldsymbol V} \xrightarrow{p}
\boldsymbol V^*
= \boldsymbol\omega\,\boldsymbol U + \boldsymbol\varepsilon^{(1)}.
\end{equation}
\end{proposition}

\paragraph{Sketch of Argument.} 

By the approximation part in Equation~\eqref{eq:approximation}, \(\mathcal{F}_n\) contains functions arbitrarily close in \(L_2\) norm to the oracle regression \(\boldsymbol r^*(\boldsymbol {X_{G^2}},\boldsymbol {X_{G}})\) defined in Equation~\eqref{eq:oracle-function}. The uniform law of large numbers ensures that the empirical minimizer \(\widehat{\boldsymbol\alpha}_n\) achieves nearly the same population risk as the best approximant. Hence 
\begin{equation}\label{eq:convergence-r}
\boldsymbol r(\boldsymbol {X_{G^2}},\boldsymbol X_G;\widehat{\boldsymbol\alpha}_n)
\xrightarrow{p}
\boldsymbol r^*(\boldsymbol {X_{G^2}},\boldsymbol {X_{G}}),
\end{equation}
which implies
\begin{equation}\label{eq:residual-limit}
\begin{split}
\widehat{\boldsymbol V}
&= \boldsymbol Y_G - \boldsymbol r(\boldsymbol X_{G^2},\boldsymbol X_G;\widehat{\boldsymbol\alpha}_n)\\
&\xrightarrow{p} \boldsymbol Y_G - \boldsymbol r^*(\boldsymbol X_{G^2},\boldsymbol X_G)
= \boldsymbol\omega\,\boldsymbol U + \boldsymbol\varepsilon^{(1)}.
\end{split}
\end{equation}


\paragraph{\textbf{\textit{Stage 2 of DIG2RSI}: Approximation and Control Function Approach.}}

\textit{Step 1: Decomposition.} Write the structural mapping for the outcome as

\begin{equation}\label{eq:M-decomp}
\begin{split}
\boldsymbol M(\boldsymbol {Y_{G}},\;\boldsymbol {X_{G}},\;\boldsymbol X,\;\boldsymbol U,\;\boldsymbol\varepsilon^{(2)})
&=
\boldsymbol M_0(\boldsymbol {Y_{G}},\;\boldsymbol {X_{G}},\;\boldsymbol X)\\
&\quad+
\boldsymbol \lambda\,\boldsymbol U
+
\boldsymbol\varepsilon^{(2)}.
\end{split}
\end{equation}

Substituting \(\boldsymbol U \approx \boldsymbol \omega^{-1}[\widehat{\boldsymbol V}-\boldsymbol\varepsilon^{(1)}]\) gives
\begin{multline}\label{eq:Y_prime}
\boldsymbol Y'
= \boldsymbol M_0(\boldsymbol {Y_{G}},\;\boldsymbol {X_{G}},\;\boldsymbol X)
  + \frac{\boldsymbol \lambda}{\boldsymbol \omega}\,\widehat{\boldsymbol V} \\
  + \underbrace{
    \boldsymbol\varepsilon^{(2)}
    - \frac{\boldsymbol \lambda}{\boldsymbol \omega}\,\boldsymbol\varepsilon^{(1)}
  }_{\boldsymbol \xi}.
\end{multline}
it satisfies $E\bigl[\boldsymbol{\xi}\mid (\boldsymbol{Y_{G}},\;\boldsymbol{G_{X}},\;\boldsymbol{X},\;\widehat{\boldsymbol{V}}\bigr]\to 0$.

\textit{Step 2: Second-Stage Estimation.}  
Let \(\mathcal{H}_n=\{h(\cdot;\boldsymbol\beta):\boldsymbol\beta\in\mathcal B_n\}\) be the second-stage function class. Assume:

\begin{assumption}[Second-Stage Approximation]\label{ass:A5}
The true function \(\boldsymbol M(\cdot)\) can be well approximated by \(\mathcal H_n\). Define the estimator
\begin{equation}\label{eq:beta_hat}
  \widehat{\boldsymbol \beta}_n
  = \arg\min_{\boldsymbol \beta\in\mathcal B_n}
    \widehat{\boldsymbol R}_{2,n}(\boldsymbol \beta),
\end{equation}
with empirical risk
\begin{equation}\label{eq:second_stage_empirical_risk}
\widehat{\boldsymbol R}_{2,n}(\boldsymbol\beta)
= \frac{1}{n}
\bigl\|
  \boldsymbol Y'
  - h(\boldsymbol {Y_{G}},\;\boldsymbol {X_{G}},\;\boldsymbol X,\;\widehat{\boldsymbol V};\;\boldsymbol\beta)
\bigr\|_2^2.
\end{equation}
\end{assumption}

\begin{theorem}[Consistency of DIG2RSI]\label{thm:consistency}
Under Assumptions~\ref{ass:A5} and standard regularity conditions:
\begin{enumerate}
  \item The parameter spaces \(\Theta_n\) and \(\mathcal{B}_n\) are compact (i.e., closed and bounded subsets of the respective Euclidean spaces, so that minimizers exist and do not escape to infinity).
  \item Uniform laws of large numbers hold for \(\widehat{\boldsymbol R}_{1,n}\) and \(\widehat{\boldsymbol R}_{2,n}\), so that the empirical risks converge uniformly to their population analogues.
  \item The population-level minimizers are well separated (i.e., the risk has a unique well-conditioned minimizer).
  \item \(\sup_{n}\|\widehat{\boldsymbol V}-\boldsymbol V^*\|_2 = o_p(1)\), meaning that the maximum over sample sizes \(n\) of the Euclidean norm difference between the estimated and true first-stage residuals converges to zero in probability.
\end{enumerate}
Then \(\widehat{\boldsymbol\beta}_n \xrightarrow{p} \boldsymbol\beta^*\); that is, the DIG2RSI estimator is consistent for the true peer effect parameter \(\boldsymbol\beta^*\).
\end{theorem}

The \textbf{\textit{second core challenge}} is the unobserved confounder \(\boldsymbol U\), which simultaneously affects the peer treatment \(\boldsymbol {Y_G}\) (via \(\boldsymbol \omega\)) and the outcome \(\boldsymbol Y\) (via \(\boldsymbol \lambda\)). As seen in Equation~\eqref{eq:Y_prime}, \(\boldsymbol Y'\) contains no direct \(\boldsymbol U\) term; the residual \(\widehat{\boldsymbol V}\) absorbs the unobserved variation \(\boldsymbol\omega\,\boldsymbol U + \boldsymbol\varepsilon^{(1)}\), and including \(\widehat{\boldsymbol V}\) as a control function effectively adjusts for the confounding, yielding an unbiased estimate of the peer effect \(\boldsymbol\beta\).


\section*{Experiments}

In this section, we first evaluate DIG2RSI on two real-world social network datasets~\cite{li2015unsupervised}. Next, we assess its performance under varying levels of adversarial weight and unobserved confounding. Finally, we conduct a case study using the Innovation Diffusion among Physicians dataset~\cite{coleman1957diffusion} to estimate peer effects. Each experiment is repeated five times, and we report the mean and standard deviation. Detailed data descriptions, experimental settings, and parameter choices are provided in the Appendix C and D. Experiments are conducted using Python 3.8, CUDA 11.0, NVIDIA Quadro P2200, and Windows OS.

\textbf{Comparison of Algorithms.}
We compare five representative or baseline approaches for peer effect estimation, including (1) 2SLS~\cite{bramoulle2009identification}, which employs second-order neighbour features as IVs within a linear two-stage least squares framework; (2) DL-2SLS~\cite{angrist1995two}, an extended deep learning variant of 2SLS by us to the nonlinear regime; (3) FN-IV~\cite{von2019use}, which  uses the average first-order neighbour features as IVs in a linear two-stage least squares setting, controlling for covariates X in the first stage; (4) LOO~\cite{jochmans2023peer}, which constructs a “leave-one-out” subnetwork by removing all edges incident to the focal individual i; the average second-order neighbour features of this subnetwork serve as IVs in a linear 2SLS framework; and (5) CES~\cite{boucher2024toward}, a nonlinear peer-effect model based on a CES functional form.

\textbf{Evaluation Metrics.} 
To evaluate DIG2RSI and all baselines, we report the \emph{absolute bias},
\( \lvert \hat{\beta}-\beta \rvert \), and the \emph{relative bias},
\( \lvert(\hat{\beta}-\beta)/\beta\rvert \times 100\% \), where
\( \hat{\beta} \) denotes the estimated effect and \( \beta \) is the
ground truth.

\textbf{Results on Peer Effect Estimation Bias.} Table~\ref{tab:bias-bc} and~\ref{tab:bias-fli} report the absolute/relative bias and the estimated PE on the BlogCatalog (BC) and Flickr datasets~\cite{li2015unsupervised}.  
The ground truth peer effect coefficient is fixed at \(0.5\). Because these two datasets include both node features and network structure, we inject unobserved confounders when generating the outcomes; details are provided in Appendix E. We see that DIG2RSI achieves the lowest bias and the largest recovered PE (the closest estimation to the ground truth PE) on both datasets, demonstrating its superiority over all competing methods. The 2SLS baseline exhibits markedly larger biases (e.g., absolute bias $1.1226$ on Flickr) and severely underestimates PE, underscoring the necessity of the control function and adversarial components incorporated in DIG2RSI.  DL-2SLS narrows the gap but still lags behind DIG2RSI, indicating that merely replacing linear regressors with neural networks is insufficient without explicitly modelling endogeneity.

\begin{table}[!t]
  \centering
  \scriptsize
  \begin{tabular}{@{}lrrr@{}}
    \toprule
    Method &
    \multicolumn{1}{c}{Absolute Bias \(\downarrow\)} &
    \multicolumn{1}{c}{Relative Bias~(\%) \(\downarrow\)} &
    \multicolumn{1}{c}{PE} \\ 
    \midrule
    \textbf{DIG2RSI}  & \(0.1689 \pm 0.0488\) & \(33.7881 \pm 9.7652\)  & \(0.3311 \pm 0.0488\) \\
    DL-2SLS  & \(0.2487 \pm 0.0539\) & \(46.7406 \pm 10.7858\) & \(0.2854 \pm 0.0330\) \\
    2SLS     & \(0.2841 \pm 0.0000\) & \(49.8249 \pm 0.0000\)  & \(0.2659 \pm 0.0000\) \\
    \underline{FN-IV}    & \(0.2019 \pm 0.0000\) & \(40.3834 \pm 0.0000\)  & \(0.2981 \pm 0.0000\) \\
    LOO      & \(0.4905 \pm 0.0000\) & \(45.1041 \pm 0.0000\)  & \(0.1095 \pm 0.0000\) \\
    CES      & \(0.5082 \pm 0.0514\) & \(43.6333 \pm 7.3461\)  & \(0.2082 \pm 0.0270\) \\
    \bottomrule
  \end{tabular}
    \caption{Bias comparison on the \textit{BC} dataset. Best method in bold, second best underlined.}
  \label{tab:bias-bc}
\end{table}

\begin{table}[!t]
  \centering
  \scriptsize
  \begin{tabular}{@{}lrrr@{}}
    \toprule
    Method &
    \multicolumn{1}{c}{Absolute Bias \(\downarrow\)} &
    \multicolumn{1}{c}{Relative Bias (\%) \(\downarrow\)} &
    \multicolumn{1}{c}{PE} \\ 
    \midrule
    \textbf{DIG2RSI} & \(0.1968 \pm 0.0299\) & \(39.3611 \pm 5.9878\)  & \(0.3032 \pm 0.0299\) \\
    DL-2SLS & \(0.2425 \pm 0.0476\) & \(48.4986 \pm 9.5231\)  & \(0.2600 \pm 0.1274\) \\
    2SLS    & \(1.1226 \pm 0.0000\) & \(64.5203 \pm 0.0000\)  & \(0.1226 \pm 0.0000\) \\
    \underline{FN-IV}    & \(0.2172 \pm 0.0000\) & \(47.4321 \pm 0.0000\)  & \(0.2981 \pm 0.0000\) \\
    LOO     & \(0.4921 \pm 0.0000\) & \(59.4197 \pm 0.0000\)  & \(0.1792 \pm 0.0000\) \\
    CES     & \(0.3921 \pm 0.0360\) & \(52.3640 \pm 6.0570\)  & \(0.1901 \pm 0.0164\) \\
    \bottomrule
  \end{tabular}
    \caption{Bias comparison on the \textit{Flickr} dataset. Best method in bold, second best underlined.}
  \label{tab:bias-fli}
\end{table}

\textbf{Hyperparameter Impact.}
We analyzed how varying the adversarial weight \(\lambda_{\mathrm{a}}\in[0,0.1]\) influences the estimation bias, as reported in Table~\ref{tab:lambda-bias-bc} (BC) and~\ref{tab:lambda-bias-fli} (Flickr). 
On the \textit{BC} dataset, the bias is minimised at \(\lambda_{\mathrm{a}}=0.01\), indicating that a \emph{small} adversarial contribution helps regularize the first–stage regression; larger weights inject noise into the control function and eventually harm the DIG2RSI estimator. 
On the \textit{Flicker} dataset, the lowest bias occurs at \(\lambda_{\mathrm{a}}=0.02\). 
Compared with the baseline \(\lambda_{\mathrm{a}}=0\), a moderate adversarial term consistently reduces both absolute and relative bias across datasets, demonstrating the effectiveness of the discriminator of our method.

\textbf{Strength of Unobserved Confounding.} 
The analysis of different strengths of unobserved confounding is provided in Appendix~F.

\begin{table}[!t]
  \centering

  \scriptsize
  \begin{tabular}{@{}lccc@{}}
    \toprule
    $\lambda_{\rm a}$ 
      & Absolute Bias \(\downarrow\) 
      & Relative Bias (\%) \(\downarrow\) 
      & PE \\ 
    \midrule
    DIG2RSI\(_{\lambda_{\rm a}=0}\)      & \(0.1904_{\pm0.0211}\) & \(38.0850_{\pm4.2215}\) & \(0.3096_{\pm0.0211}\) \\
    DIG2RSI\(_{\lambda_{\rm a}=0.01}\)   & \(0.1689_{\pm0.0488}\) & \(33.7881_{\pm9.7652}\) & \(0.3311_{\pm0.0488}\) \\
    DIG2RSI\(_{\lambda_{\rm a}=0.02}\)   & \(0.1705_{\pm0.0484}\) & \(34.1029_{\pm9.6897}\) & \(0.3295_{\pm0.0484}\) \\
    DIG2RSI\(_{\lambda_{\rm a}=0.03}\)   & \(0.1811_{\pm0.0390}\) & \(36.2171_{\pm7.7901}\) & \(0.3189_{\pm0.0390}\) \\
    DIG2RSI\(_{\lambda_{\rm a}=0.05}\)   & \(0.1860_{\pm0.0278}\) & \(37.1982_{\pm5.5530}\) & \(0.3140_{\pm0.0278}\) \\
    DIG2RSI\(_{\lambda_{\rm a}=0.08}\)   & \(0.1891_{\pm0.0371}\) & \(37.8241_{\pm7.4246}\) & \(0.3109_{\pm0.0371}\) \\
    DIG2RSI\(_{\lambda_{\rm a}=0.1}\)    & \(0.1929_{\pm0.0414}\) & \(38.5874_{\pm8.2737}\) & \(0.3071_{\pm0.0414}\) \\
    \bottomrule
  \end{tabular}
    \caption{Bias comparison of DIG2RSI under different $\lambda_{\rm a}$ values on the \textit{BC} dataset.}
      \label{tab:lambda-bias-bc}
\end{table}

\begin{table}[!t]
  \centering
  \scriptsize
  \begin{tabular}{@{}lccc@{}}
    \toprule
    $\lambda_{\rm a}$  
      & Absolute Bias \(\downarrow\) 
      & Relative Bias (\%) \(\downarrow\) 
      & PE \\ 
    \midrule
    DIG2RSI\(_{\lambda_{\rm a}=0}\)      & \(0.2411_{\pm0.0479}\) & \(48.2165_{\pm9.5827}\) & \(0.2589_{\pm0.0479}\) \\
    DIG2RSI\(_{\lambda_{\rm a}=0.01}\)   & \(0.1980_{\pm0.0316}\) & \(39.6074_{\pm6.3294}\) & \(0.3020_{\pm0.0316}\) \\
    DIG2RSI\(_{\lambda_{\rm a}=0.02}\)   & \(0.1968_{\pm0.0299}\) & \(39.3611_{\pm5.9878}\) & \(0.3032_{\pm0.0299}\) \\
    DIG2RSI\(_{\lambda_{\rm a}=0.03}\)   & \(0.2002_{\pm0.0328}\) & \(40.0477_{\pm6.5654}\) & \(0.2998_{\pm0.0328}\) \\
    DIG2RSI\(_{\lambda_{\rm a}=0.05}\)   & \(0.1998_{\pm0.0324}\) & \(39.9599_{\pm6.4738}\) & \(0.3002_{\pm0.0324}\) \\
    DIG2RSI\(_{\lambda_{\rm a}=0.08}\)   & \(0.2066_{\pm0.0317}\) & \(41.3218_{\pm6.3451}\) & \(0.2934_{\pm0.0317}\) \\
    DIG2RSI\(_{\lambda_{\rm a}=0.1}\)    & \(0.2213_{\pm0.0289}\) & \(44.2505_{\pm5.7772}\) & \(0.2787_{\pm0.0289}\) \\
    \bottomrule
  \end{tabular}
    \caption{Bias comparison of DIG2RSI under different $\lambda_{\rm a}$ values on the \textit{Flicker} dataset.}
  \label{tab:lambda-bias-fli}
\end{table}

\paragraph{Case Study: Innovation Diffusion among Physicians}
To further validate our proposed DIG2RSI framework in real world applications, we apply it to the seminal Coleman–Katz–Menzel (CKM) “Innovation among Physicians” dataset \cite{coleman1957diffusion}. This dataset was originally collected to examine how adoption of a new drug, tetracycline, diffused through a community of physicians in four Midwestern cities. The data, available through the R package CKM \texttt{spatialprobit}, contains detailed social network information and 13 individual-level attributes for 246 physicians.

In our study, we define the outcome variable as the physician's adoption timing, specifically the month of their first tetracycline prescription. Our primary objective is to estimate the causal effect of peer prescription behavior on an individual physician’s decision to adopt the new drug. We ran the model five times and obtained an average peer effect of $0.2879 \pm 0.0819 $. This aligns with common knowledge that physicians whose friends adopt the new drug earlier are themselves significantly more likely to adopt it earlier.

\section{Conclusion}

\textbf{Summary of Contributions.} In this work, we propose DIG2RSI framework tailored for networked data, which addresses at the same time both the simultaneous feedback loops and unobserved confounders while estimating peer effects in complex, nonlinear, and high-dimensional settings. To disentangle simultaneous feedback in network interactions, we first apply a differencing operator to remove feedback loops and map the network onto a DAG representation. To correct for bias arising from unobserved confounders, we integrate a two-stage residual inclusion mechanism with instrumental variables. We further provide theoretical results establishing consistency of the proposed estimator under appropriate regularity conditions. The effectiveness of our method is validated empirically on semi-synthetic data constructed over two real network structures as well as a real-world physician innovation dataset, demonstrating superior bias correction and PE estimation in the presence of imultaneous feedback and unobserved confounders.
\paragraph{Limitations \& Future Work.} Although the DIG2RSI framework offers numerous advantages, it also has certain limitations. In particular, the current method assumes that the network structure remains unchanged during the causal inference process, which prevents its direct application to dynamically evolving or temporal networks. Future work will focus on extending our approach to temporal and evolving network settings by developing a causal framework capable of capturing the evolution of node relationships.

\bibliography{aaai2026}

\begin{thebibliography}{32}
\providecommand{\natexlab}[1]{#1}

\bibitem[{Ali-Rind, Boubaker, and Jarjir(2023)}]{ali2023peer}
Ali-Rind, A.; Boubaker, S.; and Jarjir, S.~L. 2023.
\newblock Peer effects in financial economics: A literature survey.
\newblock \emph{Research in International Business and Finance}, 64: 101873.

\bibitem[{An(2015)}]{an2015instrumental}
An, W. 2015.
\newblock Instrumental variables estimates of peer effects in social networks.
\newblock \emph{Social Science Research}, 50: 382--394.

\bibitem[{Angrist and Imbens(1995{\natexlab{a}})}]{angrist1995identification}
Angrist, J.; and Imbens, G. 1995{\natexlab{a}}.
\newblock Identification and estimation of local average treatment effects.

\bibitem[{Angrist and Imbens(1995{\natexlab{b}})}]{angrist1995two}
Angrist, J.~D.; and Imbens, G.~W. 1995{\natexlab{b}}.
\newblock Two-stage least squares estimation of average causal effects in models with variable treatment intensity.
\newblock \emph{Journal of the American statistical Association}, 90(430): 431--442.

\bibitem[{Angrist, Imbens, and Rubin(1996)}]{angrist1996identification}
Angrist, J.~D.; Imbens, G.~W.; and Rubin, D.~B. 1996.
\newblock Identification of causal effects using instrumental variables.
\newblock \emph{Journal of the American statistical Association}, 91(434): 444--455.

\bibitem[{Arcidiacono and Nicholson(2005)}]{arcidiacono2005peer}
Arcidiacono, P.; and Nicholson, S. 2005.
\newblock Peer effects in medical school.
\newblock \emph{Journal of public Economics}, 89(2-3): 327--350.

\bibitem[{Bhattacharya, Malinsky, and Shpitser(2020)}]{bhattacharya2020causal}
Bhattacharya, R.; Malinsky, D.; and Shpitser, I. 2020.
\newblock Causal inference under interference and network uncertainty.
\newblock In \emph{Uncertainty in Artificial Intelligence}, 1028--1038. PMLR.

\bibitem[{Bollen(2012)}]{bollen2012instrumental}
Bollen, K.~A. 2012.
\newblock Instrumental variables in sociology and the social sciences.
\newblock \emph{Annual review of sociology}, 38(1): 37--72.

\bibitem[{Boucher et~al.(2024)Boucher, Rendall, Ushchev, and Zenou}]{boucher2024toward}
Boucher, V.; Rendall, M.; Ushchev, P.; and Zenou, Y. 2024.
\newblock Toward a general theory of peer effects.
\newblock \emph{Econometrica}, 92(2): 543--565.

\bibitem[{Bramoullé, Djebbari, and Fortin(2009)}]{bramoulle2009identification}
Bramoullé, Y.; Djebbari, H.; and Fortin, B. 2009.
\newblock Identification of peer effects through social networks.
\newblock \emph{Journal of Econometrics}, 150(1): 41--55.

\bibitem[{Coleman, Katz, and Menzel(1957)}]{coleman1957diffusion}
Coleman, J.; Katz, E.; and Menzel, H. 1957.
\newblock The diffusion of an innovation among physicians.
\newblock \emph{Sociometry}, 20(4): 253--270.

\bibitem[{De~Luna and Johansson(2014)}]{de2014testing}
De~Luna, X.; and Johansson, P. 2014.
\newblock Testing for the unconfoundedness assumption using an instrumental assumption.
\newblock \emph{Journal of Causal Inference}, 2(2): 187--199.

\bibitem[{Forastiere, Airoldi, and Mealli(2021)}]{forastiere2021identification}
Forastiere, L.; Airoldi, E.~M.; and Mealli, F. 2021.
\newblock Identification and estimation of treatment and interference effects in observational studies on networks.
\newblock \emph{Journal of the American Statistical Association}, 116(534): 901--918.

\bibitem[{Forastiere, Del~Prete, and Sciabolazza(2024)}]{forastiere2024causal}
Forastiere, L.; Del~Prete, D.; and Sciabolazza, V.~L. 2024.
\newblock Causal inference on networks under continuous treatment interference.
\newblock \emph{Social Networks}, 76: 88--111.

\bibitem[{Gu et~al.(2025)Gu, Li, Lin, and Tang}]{gu2025peer}
Gu, X.; Li, H.; Lin, Z.; and Tang, X. 2025.
\newblock Peer effects with sample selection: an application in online job training: X. Gu et al.
\newblock \emph{Empirical Economics}, 1--28.

\bibitem[{Jochmans(2023)}]{jochmans2023peer}
Jochmans, K. 2023.
\newblock Peer effects and endogenous social interactions.
\newblock \emph{Journal of Econometrics}, 235(2): 1203--1214.

\bibitem[{Kallus, Puli, and Shalit(2018)}]{kallus2018removing}
Kallus, N.; Puli, A.~M.; and Shalit, U. 2018.
\newblock Removing hidden confounding by experimental grounding.
\newblock \emph{Advances in neural information processing systems}, 31.

\bibitem[{LeSage and Pace(2009)}]{lesage2009introduction}
LeSage, J.; and Pace, R.~K. 2009.
\newblock \emph{Introduction to spatial econometrics}.
\newblock Chapman and Hall/CRC.

\bibitem[{Li et~al.(2015)Li, Hu, Tang, and Liu}]{li2015unsupervised}
Li, J.; Hu, X.; Tang, J.; and Liu, H. 2015.
\newblock Unsupervised streaming feature selection in social media.
\newblock In \emph{Proceedings of the 24th ACM International on Conference on Information and Knowledge Management}, 1041--1050.

\bibitem[{Li et~al.(2019)Li, Ding, Lin, Yang, and Liu}]{li2019randomization}
Li, X.; Ding, P.; Lin, Q.; Yang, D.; and Liu, J.~S. 2019.
\newblock Randomization inference for peer effects.
\newblock \emph{Journal of the American Statistical Association}.

\bibitem[{Liu(2014)}]{liu2014identification}
Liu, X. 2014.
\newblock Identification and efficient estimation of simultaneous equations network models.
\newblock \emph{Journal of Business \& Economic Statistics}, 32(4): 516--536.

\bibitem[{Liu(2020)}]{liu2020gmm}
Liu, X. 2020.
\newblock GMM identification and estimation of peer effects in a system of simultaneous equations.
\newblock \emph{Journal of Spatial Econometrics}, 1(1): 1.

\bibitem[{Lorenz et~al.(2020)Lorenz, Boda, Salikutluk, and Jansen}]{lorenz2020social}
Lorenz, G.; Boda, Z.; Salikutluk, Z.; and Jansen, M. 2020.
\newblock Social influence or selection? Peer effects on the development of adolescents’ educational expectations in Germany.
\newblock \emph{British Journal of Sociology of Education}, 41(5): 643--669.

\bibitem[{Manski(1993)}]{manski1993identification}
Manski, C.~F. 1993.
\newblock Identification of endogenous social effects: The reflection problem.
\newblock \emph{The review of economic studies}, 60(3): 531--542.

\bibitem[{Ogburn, Shpitser, and Lee(2020)}]{ogburn2020causal}
Ogburn, E.~L.; Shpitser, I.; and Lee, Y. 2020.
\newblock Causal inference, social networks and chain graphs.
\newblock \emph{Journal of the Royal Statistical Society Series A: Statistics in Society}, 183(4): 1659--1676.

\bibitem[{O'Malley et~al.(2014)O'Malley, Elwert, Rosenquist, Zaslavsky, and Christakis}]{o2014estimating}
O'Malley, A.~J.; Elwert, F.; Rosenquist, J.~N.; Zaslavsky, A.~M.; and Christakis, N.~A. 2014.
\newblock Estimating peer effects in longitudinal dyadic data using instrumental variables.
\newblock \emph{Biometrics}, 70(3): 506--515.

\bibitem[{Pearl(2009)}]{pearl2009causality}
Pearl, J. 2009.
\newblock \emph{Causality}.
\newblock Cambridge University Press.

\bibitem[{Sargan(1958)}]{sargan1958estimation}
Sargan, J.~D. 1958.
\newblock The estimation of economic relationships using instrumental variables.
\newblock \emph{Econometrica: Journal of the econometric society}, 393--415.

\bibitem[{Tchetgen~Tchetgen, Fulcher, and Shpitser(2021)}]{tchetgen2021auto}
Tchetgen~Tchetgen, E.~J.; Fulcher, I.~R.; and Shpitser, I. 2021.
\newblock Auto-g-computation of causal effects on a network.
\newblock \emph{Journal of the American Statistical Association}, 116(534): 833--844.

\bibitem[{Terza, Basu, and Rathouz(2008)}]{terza2008two}
Terza, J.~V.; Basu, A.; and Rathouz, P.~J. 2008.
\newblock Two-stage residual inclusion estimation: addressing endogeneity in health econometric modeling.
\newblock \emph{Journal of health economics}, 27(3): 531--543.

\bibitem[{Vapnik(1999)}]{vapnik1999overview}
Vapnik, V.~N. 1999.
\newblock An overview of statistical learning theory.
\newblock \emph{IEEE transactions on neural networks}, 10(5): 988--999.

\bibitem[{von Hinke, Leckie, and Nicoletti(2019)}]{von2019use}
von Hinke, S.; Leckie, G.; and Nicoletti, C. 2019.
\newblock The use of instrumental variables in peer effects models.
\newblock \emph{Oxford Bulletin of Economics and Statistics}, 81(5): 1179--1191.

\end{thebibliography}

\end{document}